\documentclass[reprint,jcp,floatfix,superscriptaddress,numerical]{revtex4-2}

\usepackage[utf8x]{inputenc}
\usepackage[english]{babel}
\usepackage{graphicx}
\usepackage{amsmath}
\usepackage{amssymb}
\usepackage{booktabs}
\usepackage{mathtools}
\usepackage{multirow}
\usepackage{siunitx}
\usepackage[usenames,dvipsnames]{xcolor}
\usepackage{chemfig}
\usepackage{amsthm}
\usepackage{amssymb}
\usepackage{scrextend}
\usepackage{fancyhdr} 
\usepackage{tabularx}
\usepackage{threeparttable}
\usepackage{float}
\makeatletter
\let\newfloat\newfloat@ltx
\makeatother
\usepackage{algpseudocode}
\usepackage{algorithm}
\usepackage{bbm}
\usepackage{todonotes}

\DeclareSIUnit\angstrom{\text {Å}}
\DeclareSIUnit\atom{atom}

\usepackage[hidelinks]{hyperref}

\makeatletter
\let\ps@titlepage\ps@plain
\makeatother

\DeclareSIUnit{\calorie}{\text{cal}}
\DeclareSIUnit{\kcal}{\kilo\calorie}
\DeclareSIUnit\angstrom{\text {Å}}

\newtheorem{theorem}{Theorem}%
\newtheorem{proposition}[theorem]{Proposition}% 
\newtheorem{corollary}[theorem]{Corollary}

\renewcommand{\selectlanguage}[1]{}

\begin{document}

\title{Feynman-Kac-Flow: Inference Steering of Conditional Flow Matching to an Energy-Tilted Posterior}

\author{Konstantin Mark}
\author{Leonard Galustian}
\author{Maximilian P.-P. Kovar}
\author{Esther Heid}
\email{esther.heid@tuwien.ac.at}
\affiliation{Institute of Materials Chemistry, TU Wien, A-1060 Vienna, Austria}

%-----------------------------%
% --------- DOCUMENT -------- %
%-----------------------------%

\begin{abstract}
Conditional Flow Matching(CFM) represents a fast and high-quality approach to generative modelling, but in many applications it is of interest to steer the generated samples towards precise requirements. While steering approaches like gradient-based guidance, sequential Monte Carlo steering or Feynman-Kac steering are well established for diffusion models, they have not been extended to flow matching approaches yet. In this work, we formulate this requirement as tilting the output with an energy potential. We derive, for the first time, Feynman-Kac steering for CFM. We evaluate our approach on a set of synthetic tasks, including the generation of tilted distributions in a high-dimensional space, which is a particularly challenging case for steering approaches. We then demonstrate the impact of Feynman-Kac steered CFM on the previously unsolved challenge of generated transition states of chemical reactions with the correct chirality, where the reactants or products can have a different handedness, leading to geometric constraints of the viable reaction pathways connecting reactants and products. Code to reproduce this study is avaiable open-source at \href{https://github.com/heid-lab/fkflow}{https://github.com/heid-lab/fkflow}.
\end{abstract}

\maketitle

\section{Introduction}\label{sec1}
Since its introduction by Lipman et al. \cite{lipman_flow_2023}, Conditional Flow Matching (CFM) has seen several interesting applications,  ranging from image \cite{lipman_flow_2023}, audio \cite{guo2024voiceflow} and video \cite{davtyan2023efficient} generation to decision-making \cite{zheng2023guided}, time series modelling \cite{tamir2024conditional}, protein modelling \cite{stark2023harmonic,jing2024alphafold} or molecular structure design \cite{tian2024equiflow}, amongst others. 
CFM transforms samples from a source distribution (such as random noise) to samples following a given target distribution (such as images or molecular structures) by modelling probability paths via vector fields. It largely
 improves on diffusion-based methods both in quality and speed,
 establishing CFM as a popular generative method \cite{lipman_flow_2023}.

For many current applications, the ability to adjust the distributions produced by a CFM model to conform to certain properties has gained importance. This can range from improving image quality for image generation tasks to making chemical structure geometries more physically plausible. Such a task is generally formulated as tilting a model's output distribution with an energy potential- or reward-function.
There are several approaches to fine-tuning diffusion and flow matching models to produce samples related to some reward function. Primarily, diffusion allows direct back-propagation over the sampling path \cite{clark_directly_2024}. In a different direction, stochastic optimal control-based frameworks have been used to fine-tune a model to sample from a tilted distribution, both for diffusion \cite{zhang_path_2022} and flow matching \cite{domingo-enrich_adjoint_2025, tang_branched_2025}.
Inference-only steering methods for diffusion models generally fall either under gradient based guidance \cite{bansal_universal_2023} or sequential Monte Carlo(SMC) steering \cite{wu_practical_2023, li_derivative-free_2024}. Using gradients in the guidance generally leads to a faster convergence but only works for differentiable reward functions. The SMC methods have been unified and extended into one framework through Feynman-Kac steering (FK)\cite{singhal_general_2025}. FK in particular also extends Importance Sampling (IS) which can be used trivially on any distribution to generate tilted samples, although the latter often struggles with distributions in high-dimensional and multi-modal settings. Yet, FK steering has not been extended to CFM, and is currently only used to steer diffusion models.

A particularly promising but challenging application for steering methods is the generation of chemical structures, with a breadth of physics-inspired restrictions to tilt the target distribution to samples with low energies, bond length or angle constraints, or correct chirality, where many relevant molecular structures can occur left or right-handed causing a difference in biological activity. (Non-steered) generative models have shown large success in conformer generation \cite{xu_geodiff_2022, hassan_et-flow_2024, jing_torsional_2023}, the task of generating geometries of molecules that lie at local energy minima. This success has been extended to transition state (TS) prediction or refinement, first using diffusion\cite{kim_diffusion-based_2024} and then using flow matching \cite{duan_optimal_2025, galustian_goflow_2025, darouich_adaptive_2025} models, the latter offering most importantly much faster inference time.
The generation of transition state structures is a highly relevant task in chemical reaction modelling. TSs are first order saddle points that lie on a minimal energy reaction pathway on a potential energy surface. They are normally determined through quantum mechanical calculations that are expensive to compute. CFM is able to model the generation of geometries as a transformation of probability distributions. 
 
Generating geometries that conform to the right chirality has been a tricky task in conformer generation. Chiral molecules cannot be transformed into their mirror image by rotation and translation alone, which presents a difficulty for the rotation-equivariant architectures used for these generative tasks. Simple choices to handle the problem of chirality in conformer generation involve simply flipping a conformation if a geometry with the wrong chirality was predicted \cite{ganea_geomol_2021, hassan_et-flow_2024}, however this technique breaks down as soon as multiple chiral centres are involved and particularly for diastereomers.
Ref. \cite{hassan_et-flow_2024} tried implementing proper chirality handling in a CFM setting by implementing an SO(3)-equivariant architecture, but found that it was outperformed by simple flips. There are several methods specific to graph neural networks to deal with tetrahedral chirality by reducing the invariance of message passing aggregation to be aware of the order of neighbours \cite{pattanaik_message_2020, gainski_chienn_2023} but they have failed to gain traction for representations used in generation tasks. Ref. \cite{pattanaik_message_2020} mentions inadequacy of many of the datasets including QM9 \cite{ramakrishnan_quantum_2014} for judging performance of models on complex chiral tasks. Datasets like GEOM \cite{axelrod_geom_2022} based on QM9 that are extensively used for benchmarking conformer generation tasks could, by extension, be a reason for a lack of focus on proper chirality awareness in recent work on conformer generation \cite{wang_swallowing_2024, ganea_geomol_2021, hassan_et-flow_2024, hong_accelerating_2025}, although there are methods to accurately model chirality by working in non-euclidean spaces \cite{jing_torsional_2023}.
In the more complex task of biomolecular modelling there have been recent advances in chirality awareness. Ref. \cite{wohlwend_boltz-1_2025} uses Feynman-Kac steering to reduce hallucinations including incorrect stereochemistry from a diffusion generative model. Ref. \cite{corley_accelerating_2025} improve accuracy in chiral centres by adding chirality based angle gradients as features.

The difficulty of generating chirality-aware geometries is exacerbated in the setting of TS generation. Primarily, chirality is not well-defined for TS, but rather comes together as a superposition of the chiralities of reactant and product that subtly depends on the mechanisms behind a given reaction. As such the technique of flipping into position is out of question but for simple cases where chiral centres are not actively involved in the reaction. 
Notably Ref. \cite{duan_optimal_2025} relies on incorporating reactant and product conformers into the prior distributions, thereby simultaneously alleviating many of the stereochemical difficulties. Chirality, however, will remain problematic when scaling existing TS methods up in application scope and molecule size. We therefore choose to solve the problem of generating correct isostereomers in isolation.

In this work, we introduce for the first time a framework for using Feynman-Kac steering (FK)\cite{singhal_general_2025} on CFM models to steer predictions to a tilted distribution. 
Concurrently to our research, Ref. \cite{feng_guidance_2025} recently developed guidance of CFM models towards a tilted distribution by Monte Carlo-sampling a guidance vector field. Our approach falls outside of the framework defined in their work, and we include an experimental comparison and highlight the difference in methodology. 
We then use our new FK-CFM framework to provide the first generative model that can produce stereochemically correct TSs from 2D graph information alone. This is achieved by steering GoFlow \cite{galustian_goflow_2025} with a potential based on tetrahedral volumes around tetrahedral chiral centres. We also provide an experimental analysis of FK's ability to isolate modes in a high-dimensional setting, a difficulty that was explicitly mentioned in \cite{feng_guidance_2025}.

\section{Methods}

We will first introduce general CFM notions in Section \ref{sec:cfm}. Section \ref{sec:sde} discusses the equivalence of flow models with a type of Itô stochastic differential equation (SDE). In Section \ref{sec:fks} we present Feynman-Kac steering as it applies to this setting.

\begin{figure*}[t]
    \centering
    \makebox[\linewidth]{\includegraphics[width=0.9\linewidth]{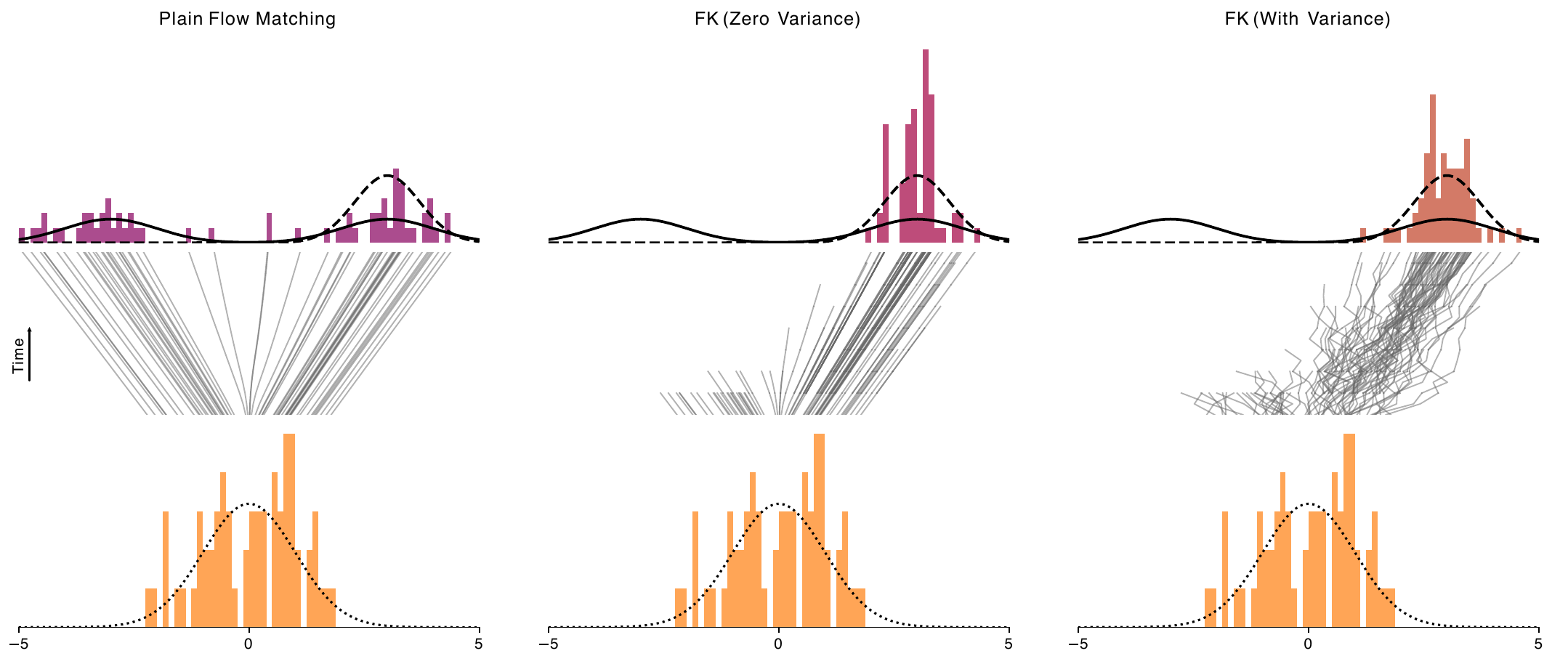}}
    \caption{Illustration of FK on a toy flow matching model from a Gaussian distribution (bottom) towards a mixture of two Gaussians (top). Tilting with a energy potential function isolates one of the modes (dashed line). FK on a deterministic Flow model converges to the correct tilted distribution but leads to a collapse in the number of active paths and a lack of exploration (middle). This is alleviated by injecting a diffusion process that maintains the marginals by correcting the flow velocities (right).}
    \label{fig:twopeaks}
\end{figure*}
\subsection{Conditional Flow Matching} \label{sec:cfm}

\textit{Conditional Flow Matching} \cite{lipman_flow_2023} trains a model to transform a known source distribution $p_0$ over $\mathbb R^d$ (often chosen to be $\mathcal N(0, I)$) to a target distribution $p_1$. This is done by seeing the transition from $p_0$ to $p_1$ as a a time-continuous transformation defining a probability path $p_t, t\in [0,1]$ and modelling this transformation with a time-dependent diffeomorphism $\psi_t(x)\in C^\infty([0,1]\times \mathbb R^d, \mathbb R^d$ on the associated random variables:
\begin{equation*}
    X_t = \psi_t(X_0) \sim p_t,\qquad X_0\sim p_0.
\end{equation*}
Such a function $\psi_t(x)$ is called \textit{flow}. A flow can be uniquely determined by a velocity field via the flow ODE
\begin{eqnarray*}
    \frac{\partial \psi_t(x)}{\partial t} &=& \mathbf v_t(\psi_t(x)),\\
    \psi_0(x) &=& x.
\end{eqnarray*}
We say that $\mathbf v_t(x)$ \textit{generates} the probability path $p_t$.\\
The \textit{Mass Conservation Formula}\cite{villani_optimal_nodate} states that $\mathbf v_t$ generates $p_t$ if and only if it satisfies the Continuity Equation \begin{equation}
    \frac{\partial p_t}{\partial t}(x) = - \nabla\cdot (\mathbf v_t(x)p_t(x))
\end{equation}
for $t\in [0,1)$.\\
How a probability path and thus a flow is chosen for a pair of distributions $(p_0, p_1)$ is often a modelling choice, but commonly an interpolation of the corresponding random variables is used:
\begin{equation}
    X_t = \alpha_t X_1 + \beta_t X_0,
\end{equation}
for $X_0\sim p_0, X_1\sim p_1$ and \textit{schedules} $\alpha_t, \beta_t$. To ensure consistency with the goal of transforming $p_0$ to $p_1$, these schedules must be chosen such that $\alpha_1 = \beta_0 = 1$ and $\alpha_0 = \beta_1 = 0$. For example, the \textit{optimal transport} schedule $X_t = (1-t)X_0 + tX_1$ satisfies these conditions. In this setting we say that a flow $\psi_t(x)$ or the \textit{velocity field} $\mathbf v_t(x)$ \textit{generates} the path between $p_0$ and $p_1$ with schedules $\alpha_t, \beta_t$. \\
To determine specific velocity fields in practice, a neural network $\mathbf v^\theta_t(x): [0,1]\times \mathbb R^d\to \mathbb R^d$ is trained to approximate the velocity field $\mathbf  v_t(x)$ from this flow ODE in an $L^2$-sense:
\begin{equation}\label{eqn:fm_loss}
\begin{aligned}
        \mathcal L_{\mathrm{FM}}(\theta):= \mathbb E_{t, X_t} \left\| \mathbf v_t^\theta(X_t) - \mathbf v_t(X_t)\right\|^2, \\ t\sim \mathcal U[0,1], X_t\sim p_t.
\end{aligned}
\end{equation}
This is called the \textit{Flow Matching loss}. Due to $p_1$ and therefore also $\mathbf v_t$ being generally unknown this objective is generally intractable. Remarkably however, one can train on a tractable objective that differs from the Flow Matching loss only by a constant by defining paths conditional on a target distribution sample $x_1$. The conditional probability path is the distribution of the random variable \begin{equation}
    X_{t|1} = \alpha_t x_1 +  \beta_t X_0 \sim p_{t|1}(\cdot|X_1 = x_1).
\end{equation}
It allows defining a tractable conditional velocity field $\mathbf v_t(X_{t|1}|x_1) = \frac{\mathrm d}{\mathrm dt}X_{t|1}$ and a tractable c\textit{onditional Flow Matching loss} 
\begin{equation} \begin{aligned}
        \mathcal L_{\mathrm{FM}}(\theta):= \mathbb E_{t, X_{t|1}, X_1}\left\| \mathbf v_t^\theta(X_t) - \mathbf v_t(X_{t|1}|X_1)\right\|^2, 
        \\ t\sim \mathcal U[0,1], X_1\sim p_1, X_{t|1} \sim p_{t|1}(\cdot|X_1).
\end{aligned}
\end{equation}

\subsection{Stochastic Differential Equations}\label{sec:sde}
The Feynman-Kac steering framework as introduced in \cite{singhal_general_2025} was formulated only for diffusion models. Technically, the similarity in situation does allow direct application of the steering to the CFM, but the deterministic nature of the flow would induce a reduction in available particles at every resampling step. To combat this, we aim to inject stochasticity into the flow model. Fig.~\ref{fig:twopeaks} depicts a non-steered CFM model versus FK steering with and without stochasticity, to illustrate the difference.\\
We will express flow and stochastic flow models as Itô processes defined by the stochastic differential equation (SDE)
\begin{equation}\label{eq:sde_standard}
    \mathrm d X_t = \mathbf v_t(X_t)\mathrm dt + \mathbf U_t(X_t)\mathrm \cdot \mathrm dB_t.
\end{equation}
Here $\mathbf v_t(x)\in \mathbb R^n$ is a drift vector, $\mathbf U_t(x) \in R^{n\times m}$ is a diffusion matrix and $B_t\in \mathbb R^m$ is a Brownian motion. Most importantly, $\mathbf U(X_t, t)$ and thus $X_t$ must be adapted to the filtration generated by $B_t$. Equation \eqref{eq:sde_standard} is understood in the sense that the process $X_t: [0,1]\times\mathbb R^n \to \mathbb R^n$ fulfils an integral equation involving an Itô integral with respect to $B_t$ given by
\begin{equation}\label{eq:sde_integral}
    X_t = \int_0^t\mathbf v_s(X_s)\mathrm ds + \int_0^t\mathbf U_s(X_s)\mathrm \cdot \mathrm dB_s.
\end{equation}
In the particular setting of a CFM, the corresponding SDE is a purely deterministic ODE i.e. $ \mathrm d X_t = \mathbf v_t(X_t)\mathrm dt$. Under the assumption of a Gaussian prior, we can inject some diffusion into the process and still maintain the same marginals through a correction term in the velocity vectors. For simplicity, we will only treat isotropic diffusion matrices of the form $\sigma(t)I$ with $\sigma(t)$ a scalar and $I\in \mathbb R^d$ the identity matrix.

\begin{theorem}[{\cite{maoutsa_interacting_2020}}]\label{thm:score_injection}
Assume $\mathbf v_t(x)$ locally Lipschitz generates the path $p_t$ between $p_0$ and $p_1$ with schedules $\alpha_t, \beta_t$. Then for any noise schedule $\sigma: [0,1]\to \mathbb R_{\geq 0}$, the stochastic process $Y_t$ defined by the SDE \begin{align*}
\mathrm dY_t &= \bigg( \mathbf v_t(Y_t) + \frac{\sigma(t)^2}{2} \nabla \log p_t(Y_t) \bigg) \mathrm dt \\
&\quad + (\sigma(t)I) \mathrm dB_t, \quad Y_0 \sim p_0
\end{align*}
is distributed according to the probability path $p_t$, i.e. $Y_t \sim p_t$.
\end{theorem}
\begin{proof}
By the Mass Conservation Theorem \cite{villani_optimal_nodate}, $\mathbf v_t(x)$ generates the path $p_t$ if and only if the continuity equation \begin{equation}\label{eq:continuity}
    \frac{\partial p_t}{\partial t}(x) = - \nabla\cdot (\mathbf v_t(x)p_t(x))
\end{equation}
holds. Using the identity $\nabla \log p_t(x) = \frac1{p_t(x)}\nabla p_t$ obtained via the chain rule, we see that \begin{equation}
\begin{aligned}
        \mathbf v_t(x)p_t(x) = \bigg(\mathbf v_t(x) &+ \frac{\sigma(t)^2}2\nabla\log p_t(x)\bigg)p_t(x) 
        \\ & - \frac{\sigma^2(t)}2 \nabla p_t(x).
\end{aligned}
\end{equation}
Writing $\mathbf w_t(x) := \left(\mathbf v_t(x) + \frac{\sigma(t)^2}2\nabla\log p_t(x)\right)$ and inserting this into the continuity equation \eqref{eq:continuity}, we obtain
\begin{equation}
    \frac{\partial p_t(x)}{\partial t} = -\nabla \cdot\left(\mathbf w_t(x)p_t(x) - \frac{\sigma^2(t)}2 \nabla p_t(x)\right).
\end{equation}
This is the Fokker-Planck equation for a stochastic process $Y_t$ with drift $\mathbf w_t(x)$ and diffusion $\sigma(t)I$, implying that $Y_t\sim p_t$.
\end{proof}
In the case of a Gaussian prior, the score $\nabla\log p_t(x)$ can be expressed directly in dependence of the velocity field:
\begin{proposition}[{\cite[B.4]{domingo-enrich_adjoint_2025}, \cite[Proposition 1]{holderrieth_introduction_2025}}]\label{prop:gaussian_score}
    Assume $\mathbf v_t(x)$ locally Lipschitz generates the path $p_t$ between $p_0\sim \mathcal N(0,I)$ and $p_1$ with schedules $\alpha_t, \beta_t$. 
    Then\begin{equation}
         \nabla \log p_t(x) = \frac1{\beta_t\left(\dfrac{\dot{\alpha}_t}{\alpha_t}\beta_t - \dot{\beta}_t\right)}\left(\mathbf v_t(x)- \frac{\dot\alpha_t}{\alpha_t} x\right)
    \end{equation}
\end{proposition}

The combination of Theorem \ref{thm:score_injection} and Proposition \ref{prop:gaussian_score} lets us express the CFM ODE as an SDE with any choice of diffusion schedule.
\begin{corollary}
Assume $\mathbf v_t(x)$ locally Lipschitz generates the path $p_t$ between $p_0\sim \mathcal N(0,I)$ and $p_1$ with schedules $\alpha_t, \beta_t$. Then for any $\sigma: [0,1]\to \mathbb R_{\geq 0}$, the stochastic process $Y_t$ defined by the SDE 
\begin{equation}
\begin{aligned}
\mathrm dY_t
&= \Bigg(
    \mathbf v_t(Y_t)
    + \frac{\sigma(t)^2}{
        2\beta_t\left(\dfrac{\dot{\alpha}_t}{\alpha_t}\beta_t - \dot{\beta}_t\right)
      }
      \left(
        \mathbf v_t(Y_t) - \frac{\dot{\alpha}_t}{\alpha_t} Y_t
      \right)
  \Bigg)\mathrm dt  \\
&\quad + (\sigma(t) I)\,\mathrm dB_t,
\qquad Y_0 \sim \mathcal N(0,I)
\end{aligned}
\end{equation}
is distributed according to the probability path $p_t$, i.e. $Y_t \sim p_t$.
\end{corollary}

In the case of flow matching with non-Gaussian priors, the score needs to be obtained in a non-analytical form. This can be done for any $x_t$ by integrating the continuity equation backwards along $[0,t]$. Since this would incur heavy computational cost at each Euler-Maruyama step, one can instead integrate a batch of paths from the trained CFM model and train a separate score model on the scores obtained. Methods such as ([SF]²M) \cite{tong_simulation-free_2024} provide a way to train a flow and score model concurrently and without simulation through a stochastic regression objective.

\subsection{Feynman-Kac steering }\label{sec:fks}
Feynman-Kac steering (FK)\cite{singhal_general_2025} is a framework for steering diffusion models to a tilted distribution. By defining a reward function $U:\mathbb R^d\to \mathbb R$ and a steering temperature $\lambda\in \mathbb R$, a
diffusion model generating samples from a distribution $p_1(x)$ is tilted to produce samples from the unnormalised distribution $p_1(x)\cdot\exp(-\lambda U(x))$ during inference, with the goal of producing samples with low values of $U(x)$. This is done by concurrently simulating a set of $k$ particles and resampling them based on intermediate potentials $G_i$. While this framework has originally been formulated for diffusion models, we will describe how it can be adapted to CFM inference.

At each step $i$, $k$ particles $x_i^{[k]}$ are proposed through the model based on the previous paths $(x_0^{[k]}, \dots, x_{i-1}^{[k]})$, either through an ODE integration step of the flow model or through an Euler-Maruyama step on the SDE defined in Theorem \ref{thm:score_injection}. Based on the path of a particle $j$, potentials $G_i\left(x_0^j, x_1^j, \dots, x_i^j\right)$ are chosen. The particles $x_i^{[k]}$ are then resampled based on these potentials. These basic steps are illustrated in Figure \ref{fig:main_illustration}.

\begin{figure}[t]
    \centering
    \includegraphics[width=0.9\linewidth]{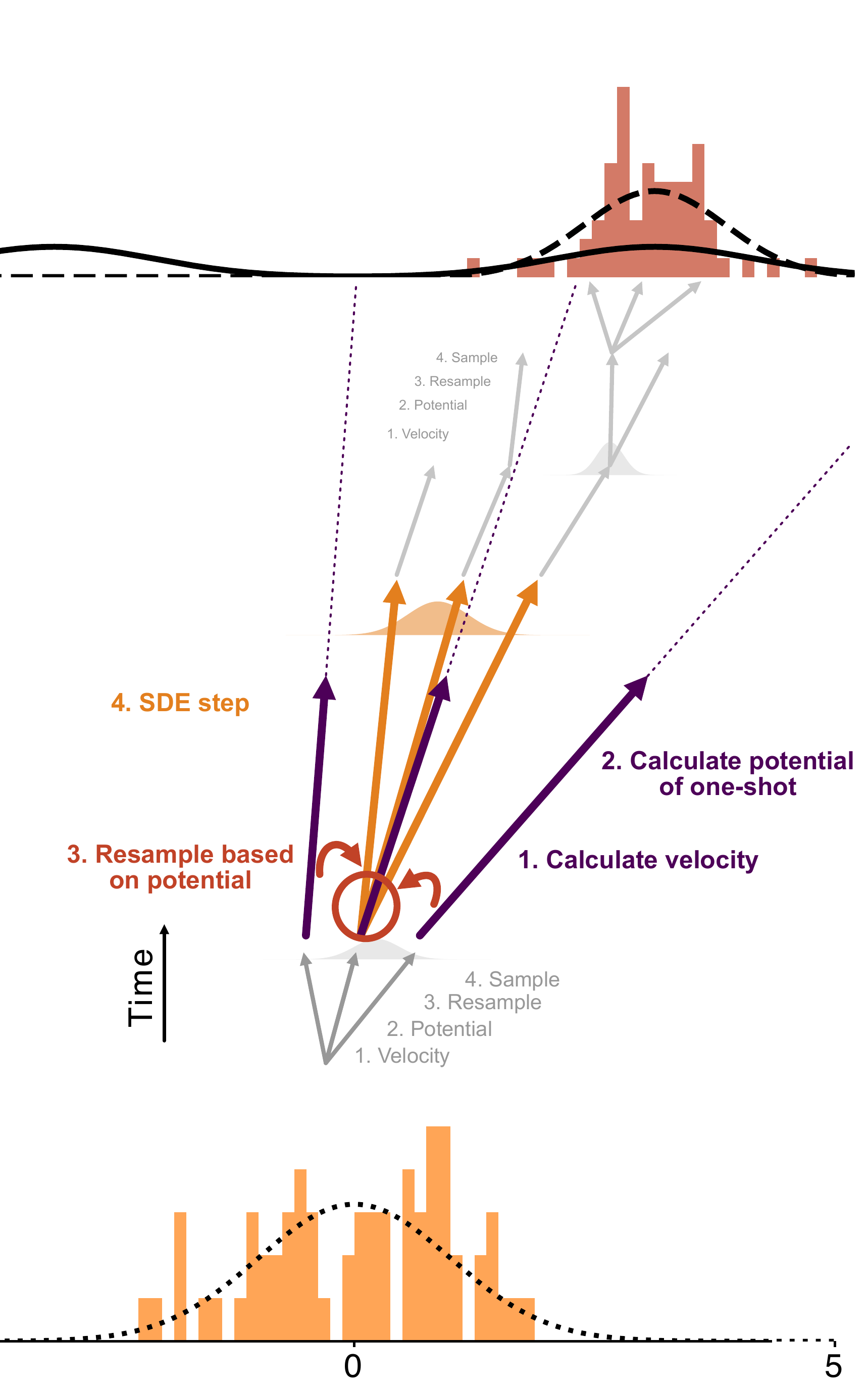}
    \caption{Illustration of the main steps in an iteration of Feynman-Kac flow matching inference. Based on the velocity vectors from a flow matching model (1.) an energy potential is evaluated at the one-shot approximation for $t=1$ (2.). The particles are then resampled (3.) before the SDE integration step (4.) which reuses the previously calculated velocities.}
    \label{fig:main_illustration}
\end{figure}

Particle systems for which along a full path $(x_0^j, x_1^j, \dots, x_n^j)$ the potentials fulfil \begin{equation}
    \prod_{i=0}^nG_i\left(x_0^j, x_1^j, \dots, x_i^j\right) = \exp\left(-\lambda U\left(x_i^j\right)\right),
\end{equation}
converge to produce samples from the tilted unnormalised distribution $p_1(x)\cdot\exp(-\lambda U(x))$ as $k\to \infty$. The specific speed at which this convergence occurs does however depend on the potentials themselves. \cite{singhal_general_2025} define first an intermediate reward function $r(x_i^j, t)$ and then define the potentials through one of three schedules:
\begin{itemize}
    \item \textbf{Difference}: $$G_i\left(x_0^j, x_1^j, \dots, x_i^j\right) = \exp(-\lambda (r(x_i^j, t_i)-r(x_{i-1}^j, t_{i-1})))$$ and $G_n = 1$.
    \item \textbf{Max}:  $$G_i\left(x_0^j, x_1^j, \dots, x_i^j\right) = \exp(-\lambda \max_{\ell=0}^i r(x_\ell^j, t_\ell))$$ and $G_n = \exp(-\lambda U(x))\cdot(\prod_{\ell=0}^nG_\ell)^{-1}$.
    \item \textbf{Sum}: $$G_i\left(x_0^j, x_1^j, \dots, x_i^j\right) = \exp(-\lambda \sum_{\ell=0}^i r(x_\ell^j, t_\ell))$$ and $G_n = \exp(-\lambda U(x))\cdot(\prod_{\ell=0}^nG_\ell)^{-1}$.
\end{itemize}
It is noted that good intermediate rewards at $x_i^j$ at step $i$ should aim to encompass the distribution of the reward for the samples obtained through $x_i^j$ as closely as possible. We will propose several flow-matching specific methods for intermediate rewards. We want to further stress that the schedules named above are generally designed to favour high-reward (i.e. low potential) regions for sampling instead of aiming to fully reproduce the distribution over the full domain. This aligns with most applications, where one specifically wants to steer towards those regions. If one wants to more accurately reflect the full tilted distribution, these schedules should be adapted to comply with this task. Particularly one should aim to avoid eliminating low-reward areas early on, e.g. by adding some tempering or annealing into the schedule.\\
The simplest way to define intermediate rewards is by evaluating the final reward function at the current point i.e. $r(x_i^j, t) = U(x_i^j)$. This is cheap since no new position need be sampled, but has only limited relevance in estimating the reward for where paths from this position will end up. This is for example how Ref.~\cite{wohlwend_boltz-1_2025} implements chirality steering.

\begin{algorithm}[t]
\caption{FK with a difference schedule}
\label{algo:FK}
\begin{algorithmic}[1]
\Require{prior \(p_0\)}
\Require{flow model $v_\theta(x,t)$}
\Require{score model $s_\eta(x,t)$}
\Require{energy potential $U(x)$}
\Require{temperature $\lambda$}
\Require{steps $K$}
\Require{diffusion variance $\sigma(t)$ }
\Require{particle size $S$}
\State{Sample $x^s \sim p_0$ for $s=1,\dots,S$}
\State{$\mathrm dt \gets 1/(K-1)$, $t \gets 0$.}
\State{compute $v^s \gets v_\theta(x^s, t)$ for all $s$.}
\State{$U_{\text{prev}}^s = 0$}
\For{$k = 0,\dots,K-2$}
\State{\textbf{Propagate} (via ODE or SDE integration):}
  \State{$\quad w^s \gets v^s + \tfrac{\sigma(t)^2}{2}\, s_\eta(x^s,t)$}
  \State{$\quad x^s \gets x^s + \mathrm dt\, w^s + \sqrt{\mathrm dt}\, \sigma(t)\, \xi^s$,
         $\xi^s \sim \mathcal{N}(0,I)$.}
\State{\textbf{Resample:}}
\State{$\quad y^s \gets x^s + (1 - t-\mathrm dt)\, v^s$}
\State{$\quad$Compute weights $G^s\gets\exp(-\lambda\cdot(U(y^s) - U_{\text{prev}}^s))$}
\State{$\quad U_{\text{prev}}^s\gets U(y^s) $}
\State{$\quad$Resample indices $i_s \sim \text{Multinomial}(\{G^s\}_{s=1}^S)$}
\State{$\quad x^s \gets x^{i_s}$}
  \State{$t \gets t + \mathrm dt$.}
\EndFor
\State{\Return $\{x^s\}_{s=1}^S$}.
\Ensure{$\{x^s_1\}_{s=1}^S$ approximately from $p_1(x)\exp(-\lambda U(x))$}
\end{algorithmic}

\end{algorithm}

CFM instead allows evaluating the reward at the Euler one-shot evaluation of the flow model starting from position $x_i^j$. That is, using $\mathbf v_{t_i}^\theta(x_i^j)$, we can calculate $y_i^j = x_i^j + (1-t_i)\mathbf v_{t_i}^\theta(x_i^j)$ as an approximation to the final value of the flow model continuing on from $x_i^j$. We then define $r(x_i^j, t) = U(y_i^j)$. Note that even though $\mathbf v_{t_i}^\theta(x_i^j)$ is used in this formula, its use does not cause any additional computational cost, since it needs to be calculated for the next ODE/SDE step, see Algorithm~1. Therefore, depending on the straightness of the paths generated by the flow matching model, this method can lead to a cheap but effective estimation of the final reward that is obtainable starting at $x_i^j$. In practice and in particular with the sum potential schedule, we often ended up multiplying this reward by an additional time dependent factor to increase steering at later inference steps when this approximation is more accurate and decrease it at early inference steps. For example, when in total $L$ resampling steps are performed, at resampling step $\ell$ we use the intermediate reward $r(x_i^j, t) = \frac1{(L+1-\ell)H_L} U(y_i^j)$ with the harmonic sum $H_L = \sum_{\ell = 1}^n\frac1{L+1-\ell}$. We call this the \textit{harmonic sum schedule} and used it throughout the experiments listed below.\\
Higher order estimations can also be used as long as no diffusion occurs between the velocity vectors used for this approximation. For example, additionally using $\mathbf v_{t_{i-1}}^\theta(x_{i-1}^j)$, one can get a second order approximation $$\tilde y_i^j = x_i^j + (1-t_i)\mathbf v_{t_i}^\theta(x_i^j) + \frac{(1 - t_i)^2}{2(t_i - t_{i-1})}(\mathbf v_{t_i}^\theta(x_i^j) - \mathbf v_{t_{i-1}}^\theta(x_{i-1}^j))$$ as long as step $i-1$ is an ODE integration step and not an SDE step.\\
For more accuracy in the potential it is also possible to compute approximations of the final sample position with more than one step at the expense of model evaluations that would not occur in the unsteered case. This is especially relevant in cases where one knows the flow to be strongly non-linear.\\
In use cases similar to GFM\cite{feng_guidance_2025} where one additionally has samples $z$ from the target distribution it is also possible to directly use importance sampling with the for most CFM methods tractable $p_t(\cdot|z)$ and a time annealed potential to resample, although this option has not been explored by us.

\section{Results}\label{results}

We first analyse FK's ability to isolate modes in a high-dimensional setting in a benchmark in \ref{sec:hypercube}. We then compare FK to Guidance Flow Matching \cite{feng_guidance_2025} on the synthetic datasets used in that work in \ref{sec:guidance}. Finally, we demonstrate the usefulness of FK in the applied setting of producing chirality-aware transition state geometries in \ref{sec:ts}.

\subsection{High-dimensional mode isolation}\label{sec:hypercube}

\begin{figure}
    \centering
    \includegraphics[width=1\linewidth]{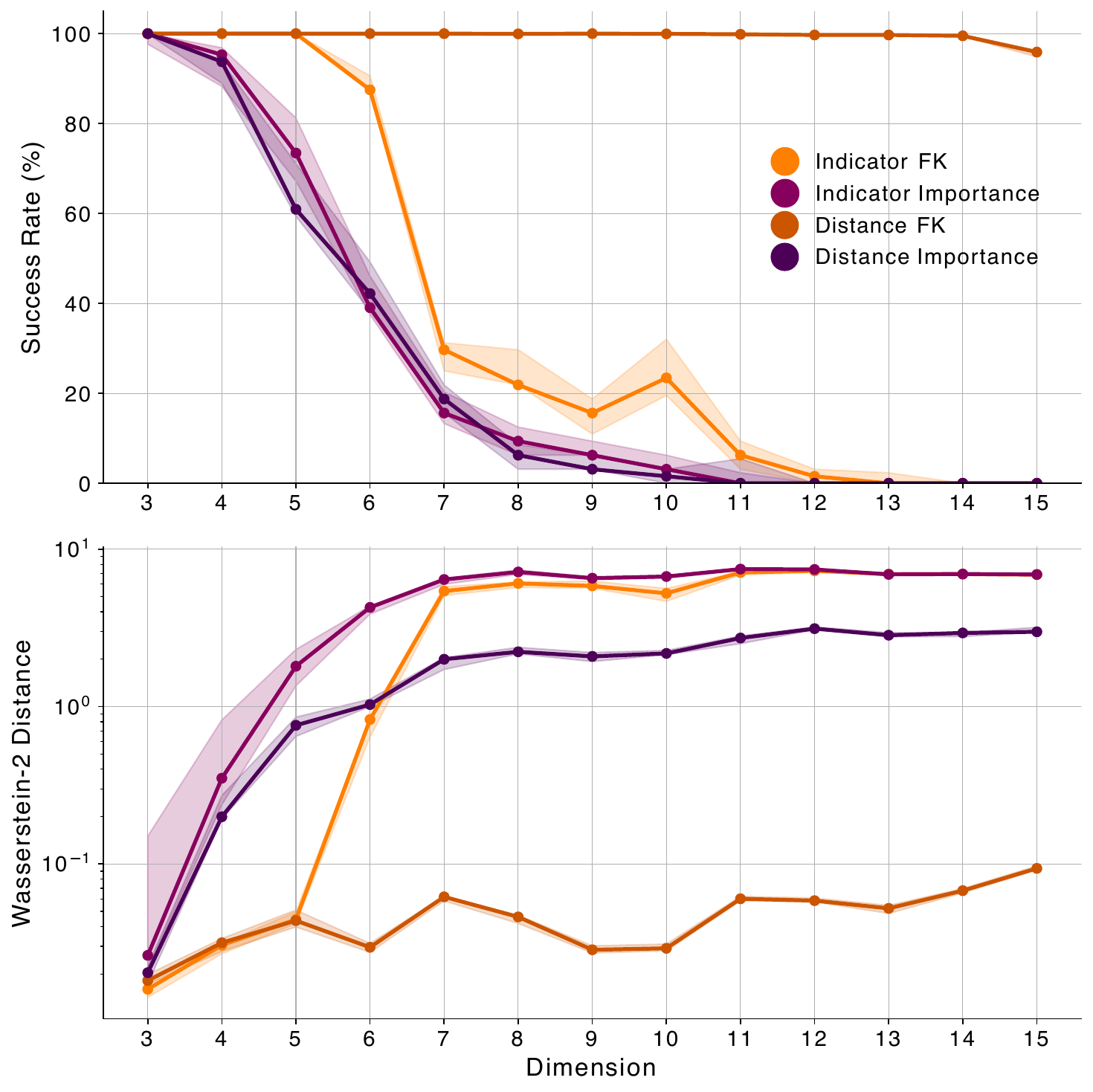}
    \caption{Success rate (top) and sliced Wasserstein-2 approximation (bottom) for FK and IS on the $d$-dimensional benchmark for 32 particles using the indicator and distance potentials. Shown are the median and quartiles over 10 runs. All inference used 50 integration steps with FK reweighing every 3 steps for FK.}
    \label{fig:hypercube_scaling}
\end{figure}
\begin{figure}
    \centering
    \includegraphics[width=1\linewidth]{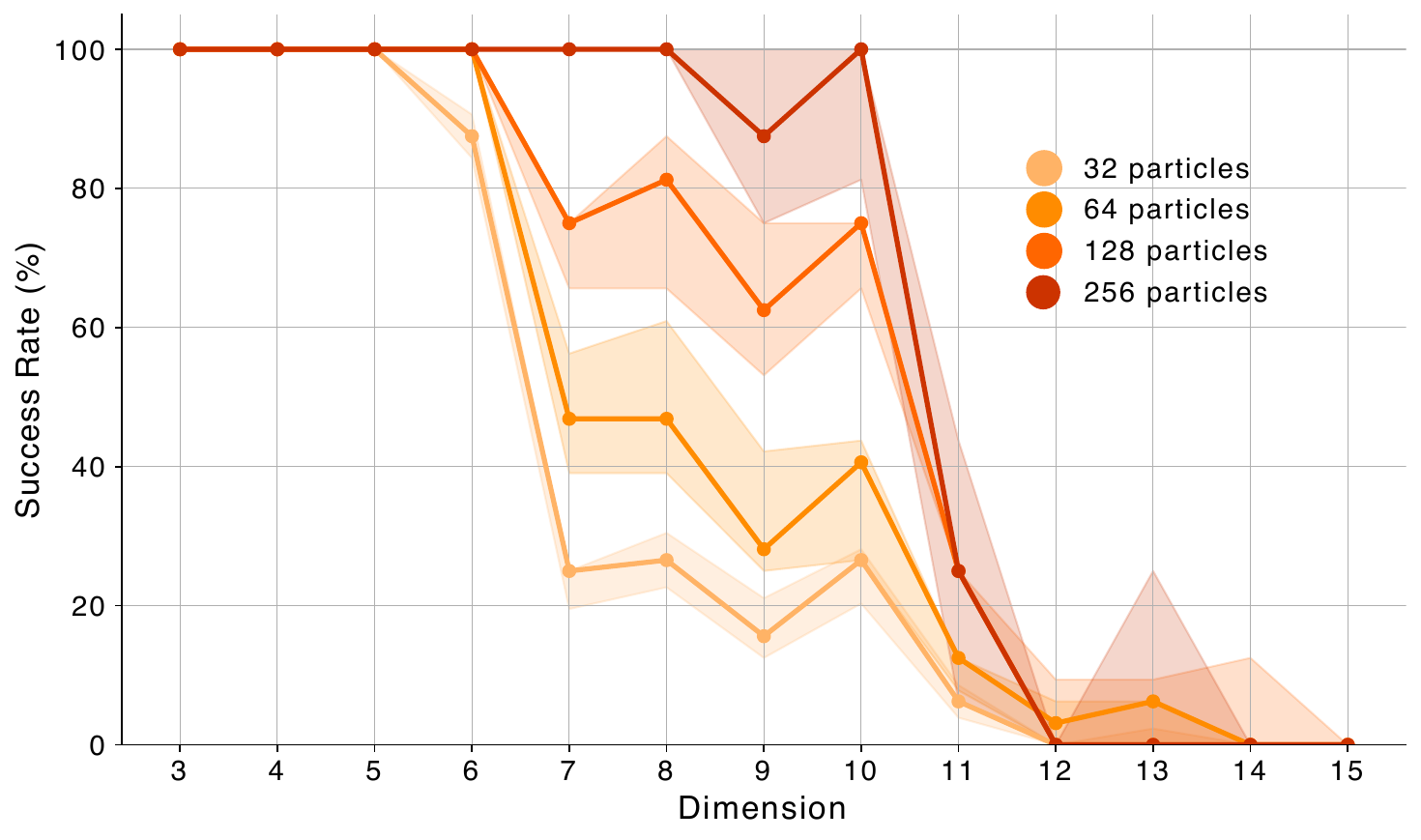}
    \caption{Success rate for FK with the indicator potential for different particle system sizes. Shown are the median and quartiles over 10 runs. Inference used 50 integration steps with FK reweighing every 3 steps.}
    \label{fig:hypercube_particles}
\end{figure}

Importance Sampling (IS) techniques deal particularly poorly with high-dimensional and multimodal settings. To test the performance of FK-CFM in such a setting, we perform the following benchmark on synthetic datasets. In $d$ dimensions, we train an ([SF]²M) \cite{tong_simulation-free_2024} model on a flow from uniform distribution on the hypercube $[-1, 1]^d$ to a distribution of isotropic Gaussians with small variance centered around the corners of the hypercube $[-2, 2]^d$. In this setting, the number of modes increases exponentially in dimension. At the same time each individual mode becomes less likely as the dimensions increase. 
\\
We try to isolate one of the Gaussians by tilting with potentials of the form $$I(x):= w\cdot\left(1-\mathbbm 1_{[0,\infty)^d}(x)\right)$$ and $$D((x_i)_{i=1}^n) = w \sum_{i=1}^d \max(0, - x_i).$$ The \textit{indicator potential} $I$ is thus a discontinuous step function that tilts mass away from $\left([0,\infty)^d\right)^C$, while the \textit{distance potential} $D$ is a continuous function that achieves the same purpose. For large values of $w$, only the Gaussian centered at $(2, \dots, 2)$ should remain. We measure the success of the steering by the relative number of samples that land in $[0,\infty)^d$ as well as a sliced Wasserstein-2 approximation of the distance to the target Gaussian. We compare FK with IS using the same number of particles.\\
We find that on the indicator potential, FK outperforms IS but suffers from a similar performance degradation once the number of modes is larger than the number of simulation particles, as shown in Fig.~\ref{fig:hypercube_particles}.  However, FK is able to use the information given through the continuous distance potential very effectively and does not suffer from the same curse of dimensionality, isolating the relevant mode with small particle sizes even in high dimensions (see Figure \ref{fig:hypercube_scaling}). This is especially noteworthy given the difficulties with high-dimensional distributions reported for guidance methods for flow matching \cite{feng_guidance_2025}, which our FK implementation overcomes.

\subsection{Comparison with Guidance}\label{sec:guidance}
We evaluated FKS on three pairs of source and target distributions of synthetic 2D datasets in alignment with \cite{feng_guidance_2025}: Circle to S-curve, uniform to 8 Gaussians and 8 Gaussians to Moons, see Fig.~\ref{fig:guidance_comp_small}. We trained a CFM model and a minibatch optimal transport CFM (OT) model \cite{tong_improving_2024} based on their configuration and evaluated Guidance Flow Matching (GFM) and FK using the same two models. For FK we used deterministic inference only without diffusion injection to avoid training a separate score model. We chose to include OT since it naturally generates straighter paths which increases the accuracy of the reward estimation. For GFM we used 10240 samples from the target distribution for the Monte Carlo sampling of the guidance field. We ran all models and steering methods with 40 integration steps on 1024 samples. The results can be seen in Figure \ref{fig:guidance_comp_small} and Table \ref{tab:guidance_comparison}. For details on the parameters used and an extended plot see the Appendix \ref{appendix:results}. As expected, using FK without stochastic inference results in less distinct samples although all 1024 are still present. We find that FK managed to achieve lower Wasserstein-2 distances than GFM without reliance on using additional target samples during inference or training an additional model. However, it is important to note that we did not perform a rigorous hyperparameter optimisation on the trained models or the inference methods and include this comparison mostly as an illustration to show that both methods solve similar tasks. In application either method can have its advantages purely from a perspective of design: as shown in Section \ref{sec:hypercube}, FK performs well on cases where only a small fraction of prior samples are relevant for the tilted distribution by eliminating superfluous samples and repopulating the important samples through stochasticity. For these situations, as also mentioned in \cite{feng_guidance_2025}, GFM struggles since it needs to transport the whole prior distribution to the target using untilted target samples. On the other hand, Guidance can be expected to perform better for cases where it is important to capture the full tilted distribution well and not just the high-reward areas, since FK is not designed for this task specifically. Furthermore it can be important in application to obtain a direct coupling between prior samples and tilted target samples, which is possible directly with GFM but would require a backwards integration for FK. 

\begin{figure}
    \centering
    \includegraphics[width=1.0\linewidth]{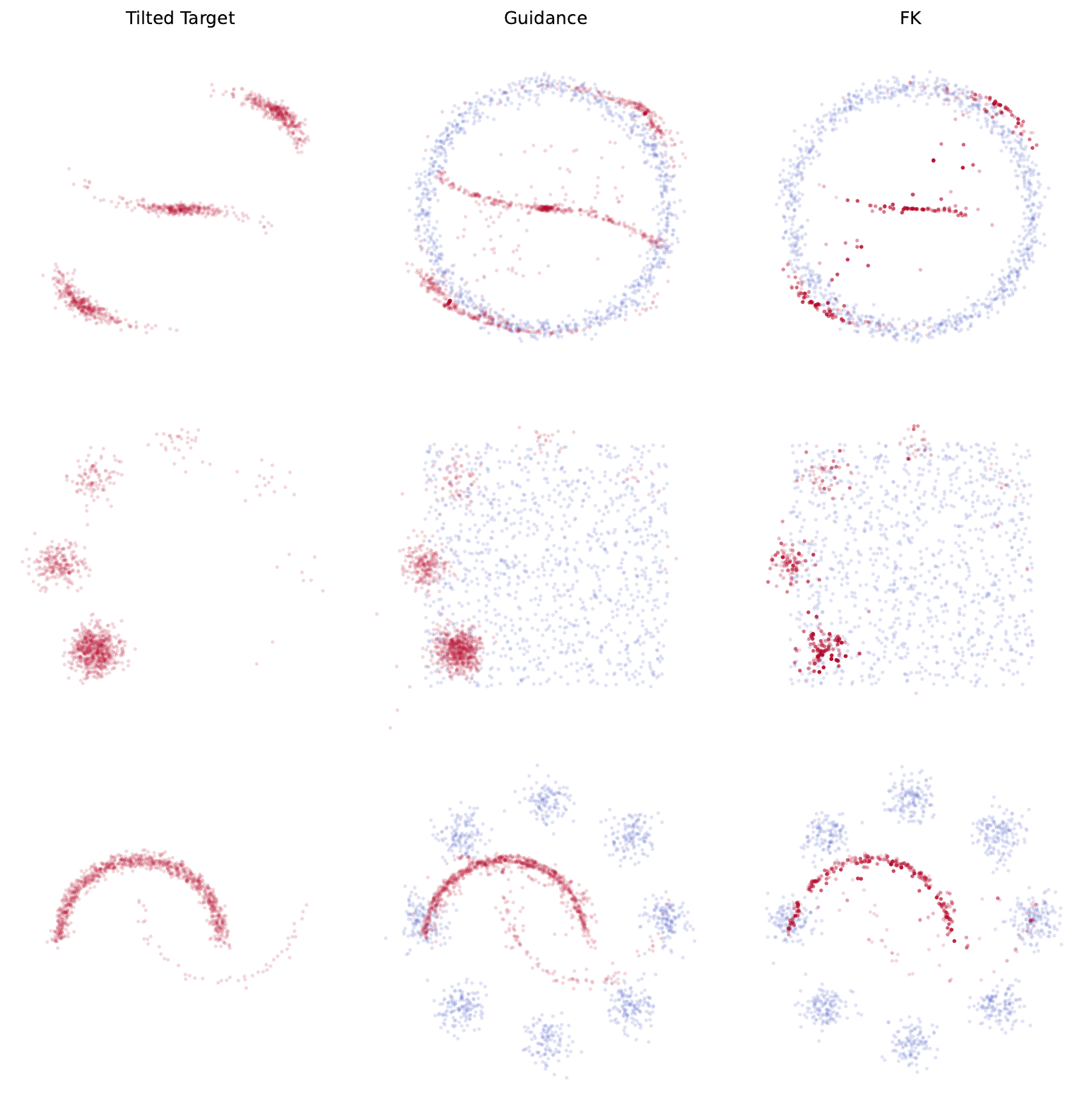}
    \caption{Distributions generate by Monte Carlo Guidance CFM\cite{feng_guidance_2025} (middle) and FK (right). The target distribution tilted with the same potential is displayed on the left.}
    \label{fig:guidance_comp_small}
\end{figure}

\begin{table}
  \centering
  \makebox[\linewidth]{%
    \begin{tabular}{l||c|c}
      \toprule
      & G CFM & FK CFM \\
      \midrule
      Circle $\rightarrow$ S
      &  \(0.337 \pm 0.020\) & \(0.274 \pm 0.073\) \\
      Uniform $\rightarrow$ 8 Gaussians
      &  \(0.178 \pm 0.028\) & \(0.137 \pm 0.032\) \\
      8 Gaussians $\rightarrow$ Moons
      &  \(0.122 \pm 0.011\) & \(0.097 \pm 0.020\) \\
      \bottomrule
    \end{tabular}%
  }
  \vspace{0.75em}
  \makebox[\linewidth]{%
    \begin{tabular}{l||c|c}
      \toprule
      &  G OT & FK OT \\
      \midrule
      Circle $\rightarrow$ S
      & \(0.244 \pm 0.019\) & \(0.216 \pm 0.059\) \\
      Uniform $\rightarrow$ 8 Gaussians
      &  \(0.343 \pm 0.007\) & \(0.132 \pm 0.042\) \\
      8 Gaussians $\rightarrow$ Moons
      & \(0.226 \pm 0.009\) & \(0.068 \pm 0.010\) \\
      \bottomrule
    \end{tabular}%
  }
  \caption{Wasserstein-2 distances to the tilted distributions for the three dataset pairs. Top: CFM with Guidance (G CFM) and Feynman Kac steering (FK CFM). Bottom: minibatch OT with Guidance (G OT) and Feynman Kac steering (FK OT). Means and standard deviations are over 10 batches of 1024 samples each with 40 integration steps. Lower is better.}
  \label{tab:guidance_comparison}
\end{table}

\subsection{Chirality-aware Transition State generation}\label{sec:ts}
\begin{figure}[h]
    \centering
    \includegraphics[width=1.0\linewidth]{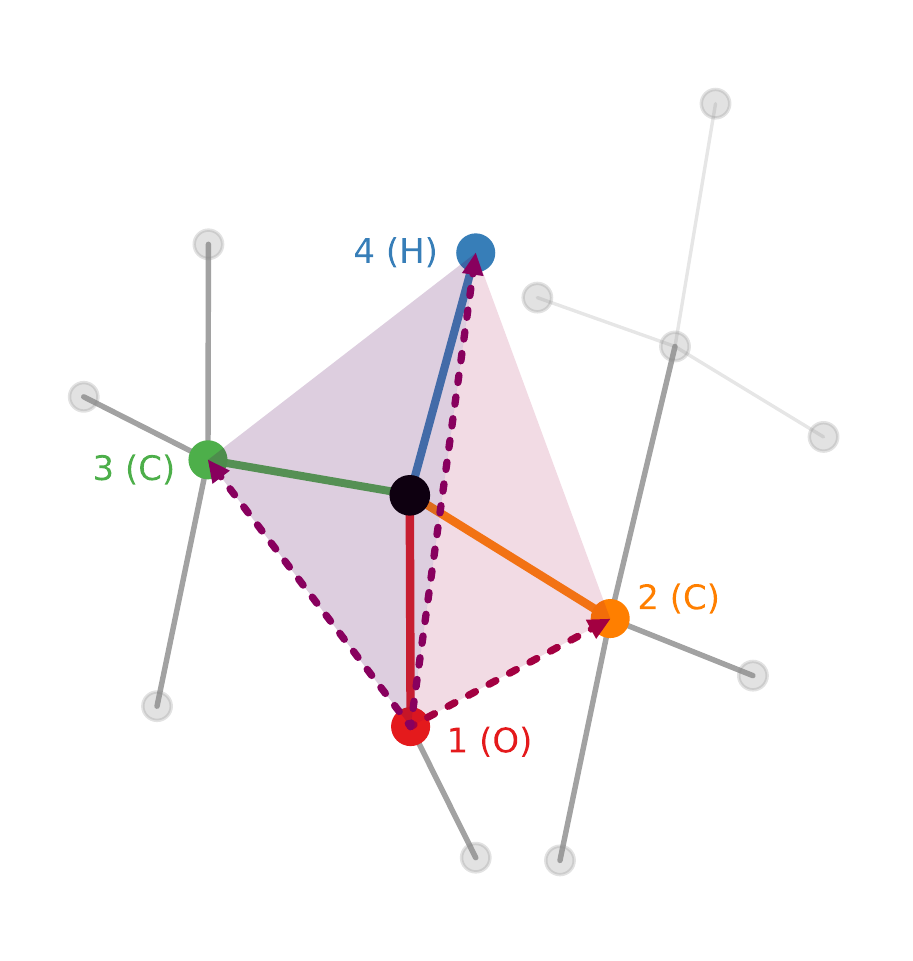}
    \caption{Tetrahedral chiral centre of 2-Butanol. The neighbours are labelled with their CIP priorities. The chiral potential is determined by the volume of the tetrahedron spanned by the centre's neighbours and their ordering. }
    \label{fig:volume}
\end{figure}

For generating TS geometries from 2D graph information alone, we use a FK approach inspired by \cite{wohlwend_boltz-1_2025} and define a potential based on tetrahedral chirality to tackle TS chirality. However, instead of defining the potential based on the dihedral angles involved in a tetrahedral chiral centre, we calculate a potential based on the directed chiral volume of the tetrahedron spanned by the neighbouring atoms based on the chirality of the reactant and product molecules. Assume $\mathbf x_1, \mathbf x_2, \mathbf x_3, \mathbf x_4$ are positions that are sorted by the CIP priorities of the neighbours of a chiral centre of either the reactants or products. Then we define the chiral volume as the triple product 
\begin{equation}\label{eq:volume}
    V(\mathbf x_1, \mathbf x_2, \mathbf x_3, \mathbf x_4) = \frac{\left[(\mathbf x_2- \mathbf x_1)\times(\mathbf x_3-\mathbf x_1)\right]\cdot(\mathbf x_4- \mathbf x_1)}6.
\end{equation}
An illustration of this can be seen in Figure \ref{fig:volume}. We then flip the sign of this volume if the chiral centre is S such that if the positions have correct chirality the signed volume is always negative. This is then passed through a Relu function to set the potential for correct centres to zero. The final potential for any geometry is the sum over all chiral centres of both reactant and product of this directed tetrahedral volume.
\\
As a base model we use GoFlow \cite{galustian_goflow_2025}, trained with the same parameters as mentioned there on RDB7 \cite{spiekermann_high_2022}. We use the same median aggregation as introduced by GoFlow for both the base model over several samples and FK over the particles. 
\\
\begin{figure}[h]
    \centering
    \includegraphics[width=1.0\linewidth]{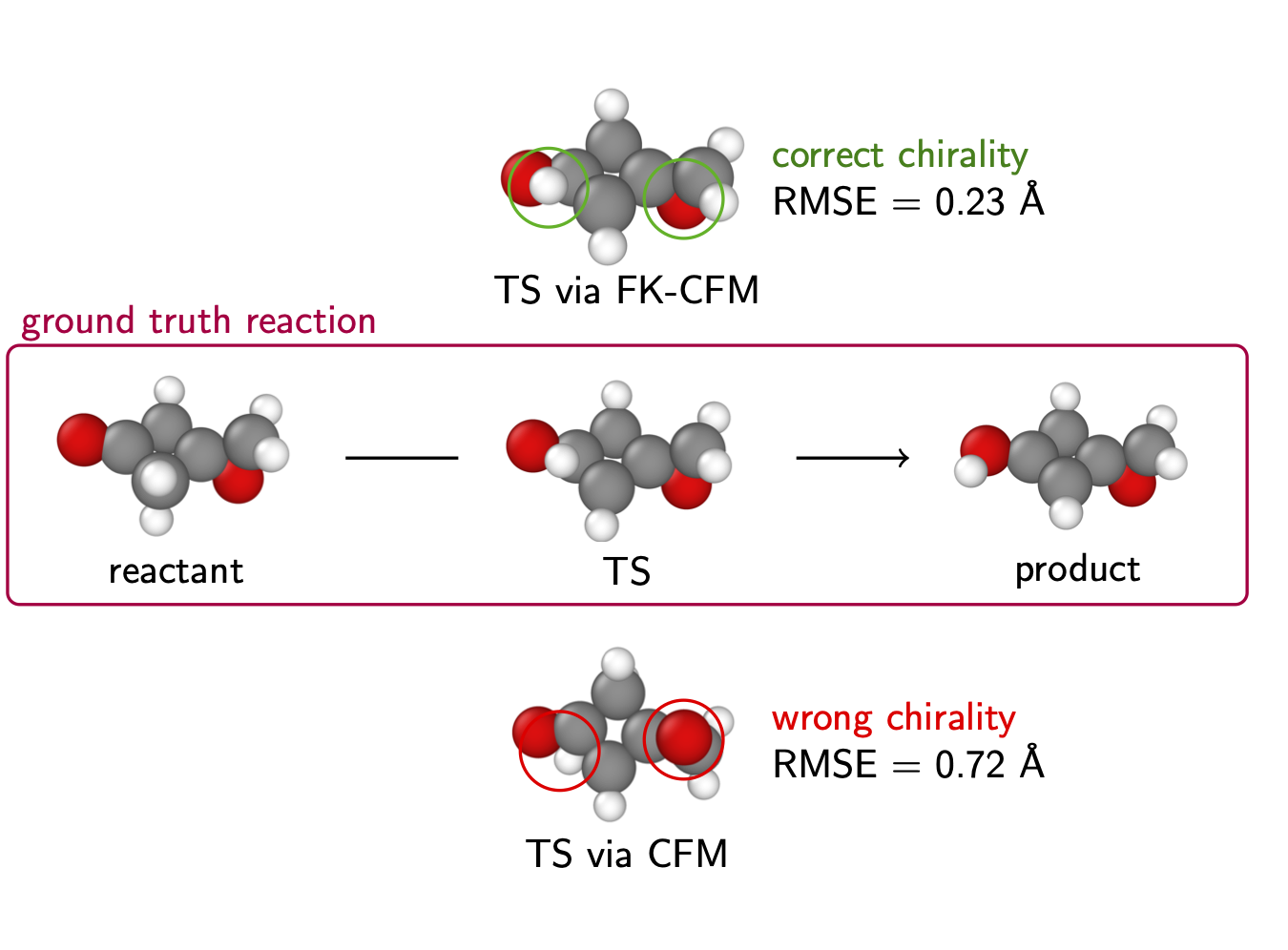}
    \caption{Exemplary ground-truth reaction, and the TS structures predicted by FK-CFM (top), and CFM (bottom), where without steering, wrong chirality in two chiral centres (one within the reaction center) is predicted.}
    \label{fig:rxn120}
\end{figure}
\begin{figure*}
    \centering
    \includegraphics[width=0.9\linewidth]{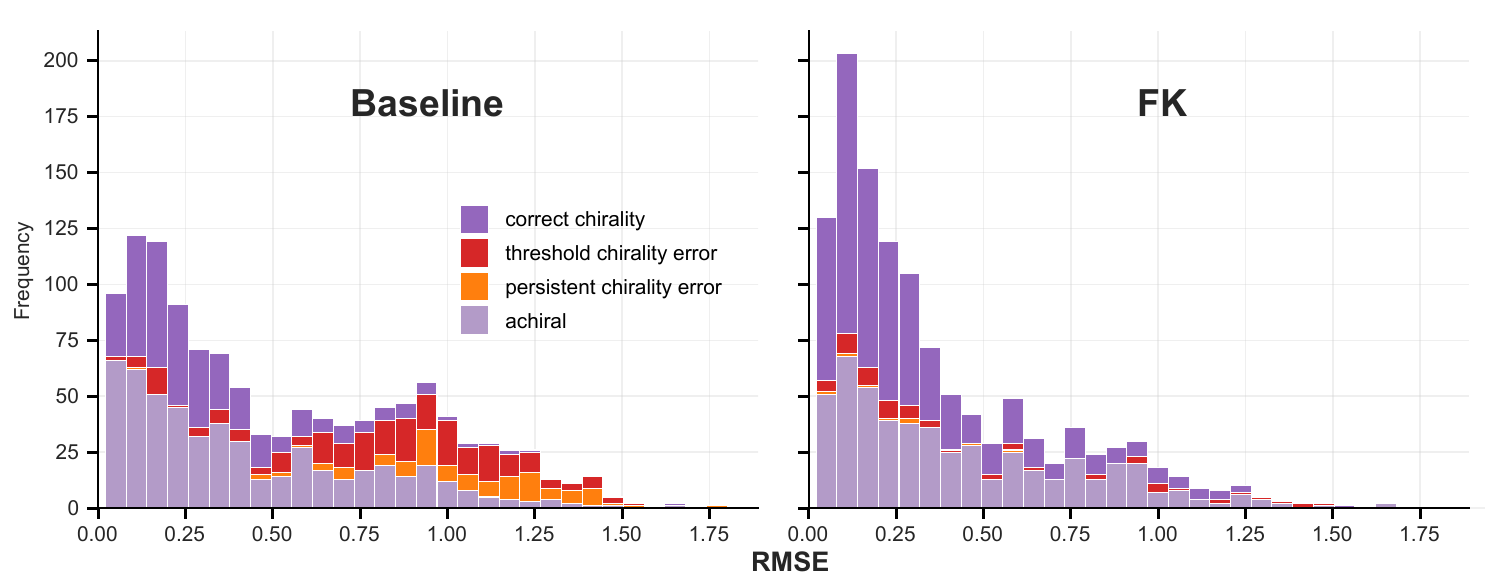}
    \caption{RMSE to the ground truth for both the base model (left) and a FK model with a chiral potential (right). The unsteered model has a clear error modality of predicting the wrong chirality, as indicated by the reactions with high RMSE that are categorised with a persistent chirality error (orange) or a threshold chirality error (red). FK eliminates this modality and overwhelmingly predicts the correct chirality.}
    \label{fig:rmse_histogram}
\end{figure*}
As noted above, the chirality of TSs is hard to classify exactly. The base model without steering has a failure mode of predicting the wrong chirality that manifests in the root mean squared error (RMSE) to the ground truth as a distinct peak. An exemplary reaction and the generated TSs is shown in Fig.~\ref{fig:rxn120}, together with the discrepancy in RMSE. As such, RMSE is a good basic measure of how well FK can eliminate predictions with the wrong chirality. We further introduce two additional metrics that help us analyse the effectiveness of the steering to eliminate wrong chiralities. \\
Firstly, we can measure the prediction accuracy directly on all chiral centres that are unambiguous since their chirality does not change during the reaction. We define a \textit{consistent chirality error} at reactions for which at least one reaction‑consistent chiral centre has a chirality opposite to that dictated by the 2D graph, based on the sign of the chiral volume as in equation \eqref{eq:volume}.\\
Secondly we want to give a score that incorporates all chiral centres. Similar to equation \eqref{eq:volume} we calculate a chiral volume but this time normalised by the total edge lengths:
\begin{equation*}
    \tilde V(\mathbf x_1, \mathbf x_2, \mathbf x_3, \mathbf x_4) = \frac{\left[(\mathbf x_2- \mathbf x_1)\times(\mathbf x_3-\mathbf x_1)\right]\cdot(\mathbf x_4- \mathbf x_1)}{\|\mathbf x_2- \mathbf x_1\|\cdot\|\mathbf x_3- \mathbf x_1\|\cdot\|\mathbf x_4- \mathbf x_1\|}.
\end{equation*}
We define a \textit{thresholded chirality error} at reactions for which there is a (not necessarily reaction consistent) chiral centre where the sign of $\tilde V$ is inconsistent with with the chirality dictated by the 2D-graph and the absolute value of $\tilde V$ exceeds a threshold $\theta$. For the results presented below, we consistently use $\theta = 0.25$.

As can be seen in Figure \ref{fig:rmse_histogram} left, the failure mode for the base model is indeed caused by predictions with the wrong chirality, both for reactions with change in their chiral centres and those without. This occurs in approximately half of the chiral reactions. Applying FK with the chiral potentials is able to remove this failure mode and predict most TS with the correct chirality. Some reactions with low RMSE and the wrong chirality remain, however this is mostly due to the inaccuracy of the thresholded chirality error calculation and the indeterminate nature of the chirality of TS: the ground truth TS geometries have a similar fraction of reactions with a thresholded chirality error (see Table \ref{tab:rmse_table}). We further find ground truth TSs with an impossible chirality, inconsistent with the reactant and product, shown in the Appendix, where clearly the data entry is wrong. This also indicates that FK does not force a flat state onto these geometries due to the base model having little density there. \\
In summary, we find that FK-CFM generates chirality-aware transition states reliably, which addresses an important but previously unresolved challenge in the field, namely creating guess structures for transition states that adhere to constraints set by the reactant and product chiralities.

\begin{table}[]
    \centering
    \begin{tabular}{l|cccc}
         & DMAE (Å)& RMSE (Å)& PCE (\%) & TCE (\%) \\\hline
         GT\,&&&0.8&6.3\\ 
        GoFlow &0.123&0.527&9.0&27.8\\
        GoFlow + FK & 0.120&0.369&0.7&5.8
    \end{tabular}
    \caption{Comparison of the base model to the chirality based FK. The chirality error mode cannot be seen in the mean absolute error of the inter-atomic distances (DMAE) but manifests itself clearly in the direct root mean squared error of the atoms (RMSE). It also shows in the error rates for the persistent chirality error (PCE) and the thresholded chirality error (TCE) which are given as a relative amount over the full test dataset. Over the dataset, 56.6\% of the reactions contain at least one tetrahedral chiral centre in either reactant or product. }
    \label{tab:rmse_table}
\end{table}

\section{Discussion}\label{discussion}
In the following, we discuss limitations of FK we encountered in this work.
The main limitation of FK as presented here is the reliance on a tractable score for the injection of stochasticity during inference. We have mentioned several methods to deal with this in practice. At the lack of such a score one can still apply FK on the deterministic model and achieve guaranteed convergence but with a reduction in number of distinct samples, although this is a trait shared with IS.\\
Future work should also concentrate on how FK can integrate with guidance methods that adapt the underlying velocity field such as GFM or steering based on the gradients of the potential or the potential estimation while preserving the marginals. This could be achieved by adapting the steering potentials at updates of the velocity field through an appropriate change of measures and Girsanov's theorem. We leave this to be explored in future work. \\
In handling chirality of TS, we focused purely on tetrahedral stereochemistry. The methods shown here open up the possibility to include further geometry based tilting, such as bond stereochemistry and steric clash. It also remains to investigate more deeply how FK deals with extreme cases where the TS lies very close to either reactant or product. The prediction of low-RMSE geometries that are characterised as wrong by the thresholded chirality error as well as the similarity in DMAE between the steered and unsteered model seem to indicate that on RDB7 this poses no problem, but \cite{corley_accelerating_2025} explicitly mention steering into out of distribution regions as a likely error source for diffusion FK as applied in \cite{wohlwend_boltz-1_2025}. This can be alleviated by stronger intermediate reward estimates or by using schedules that less aggressively try to steer towards high-reward regions, at the cost of sampling efficiency.

\section{Conclusion}\label{sec5}
We have successfully derived Feynman-Kac steering for conditional flow matching, and showcased its ability to generate samples from a tilted distribution on low and high-dimensional synthetic data, where we demonstrate how FKS performs favourably even for high-dimensional tasks with a sparse tilted distribution. We have further implemented FK steering for CFM on the challenging real-world use case of generating chemical transition state structures steered toward correct chiralities, addressing a long standing challenge in this field. We envision that FK-CFM will enable advances in several other fields of generative modelling.

\section*{Data and Software Availability}
Data and code to reproduce this study is available open-source at \href{https://github.com/heid-lab/fkflow}{https://github.com/heid-lab/fkflow}

\section*{Acknowledgments}
% ERC
Funded by the European Union (ERC-2024-STG, Project 101162908 — DeepRxn). Views and opinions expressed are however those of the author(s) only and do not necessarily reflect those of the European Union or the European Research Council. Neither the European Union nor the granting authority can be held responsible for them. 
% FWF:
This research was funded in part by the Austrian Science Fund (FWF)
[10.55776/STA192]. For open access purposes, the author has applied a CC BY public copyright license to any author accepted manuscript version arising from this submission.

\section*{Author contributions}
\textbf{KM:} Conceptualization, Formal analysis, Investigation, Methodology, Project administration, Software, Validation, Visualization, Writing - original draft.
\textbf{LG:} Conceptualization, Data curation, Investigation, Software, Validation, Writing - review \& editing.
\textbf{MK:} Conceptualization, Data curation.
\textbf{EH:} Conceptualization, Funding acquisition, Project administration, Supervision, Visualization, Writing - review \& editing.

\begin{figure*}
    \centering
    \includegraphics[width=1\linewidth]{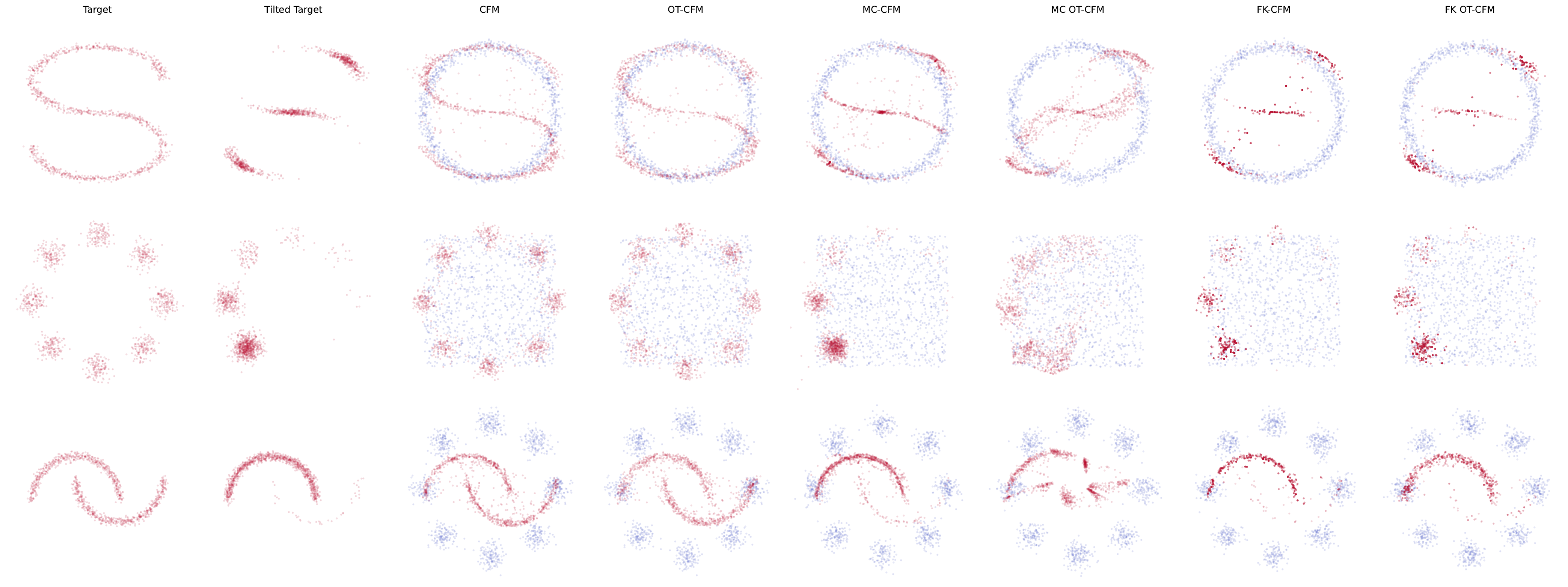}
    \caption{Evaluation of Guidance Flow Matching with Monte Carlo Guidance (MC) and deterministic Feynman-Kac steering (FK) on synthetic datasets with a trained Conditional Flow Matching model (CFM) and a Minibatch Optimal Transport model (OT-CFM). Included for comparison are also samples from the base target $p_1(x)$, the tilted target $\frac1Z p_1\exp(-J)$ and untilted CFM and OT-CFM. For each, 1024 samples are shown that were generated with 40 integration steps. GFM was calculated using 10240 MC samples. For FK, resampling was performed every 5 integration steps with a particle size of 128.}
    \label{fig:guidance_plot}
\end{figure*}

\section*{Appendix} \label{appendix}

\subsection{Experiment details}\label{appendix:results}
\subsubsection{High-dimensional mode isolation}
For each dimension $d$ we train a ([SF]²M) model on the prior and target distributions described in the text. The model architecture is a four layer MLP with a width of 128. The batch size is held constant at $512$ but we increase the number of training epochs by dimension as $1000\cdot d$. We use $\sigma =2.0$ throughout, also during inference. The Wasserstein-2 distances are approximated via $4096$ random one-dimensional projections. All runs, unless otherwise noted are run with 32 FK particles, 50 integration steps and a reweighting every 3 steps. 

\subsubsection{Comparison with Guidance}
We performed the comparison directly using training scripts as given in \cite{feng_guidance_2025}, wrapping them for FK inference. For GFM inference we directly took the hyperparameters as given where possible. For OT CFM, only the hyperparameters for the pair Uniform to 8 Gaussians was available. We used the same hyperparameters for the other two pairs, increasing the scale parameter and MC samples since this led to better results. The Wasserstein distances in Table \ref{tab:guidance_comparison} are given as mean over 10 runs of 1024 samples. For FK we used 128 particles for the resampling and resampled ever 10 steps.  One of the 10 runs can be seen for each of the three dataset combinations and each of the model/inference method combinations in Figure \ref{fig:guidance_plot}.

\subsubsection{TS generation}
We trained GoFlow with the same parameters as in \cite{galustian_goflow_2025} (5.2M parameter model). For the transformation into an SDE we assumed a Gaussian prior despite the Kabsch rotations that technically happen during training - experimentally no difference in performance occurred between the deterministic inference and SDE inference with no steering. We decreased diffusion injection linearly in time with a base variance of $0.3$ at $t=0$ and a variance of $0.0$ at $t=1$. We used a steering temperature of $\lambda =0.4$. 

\subsubsection{Inconsistent ground truth data}
For some reactions in RDB7, we find that the TS chirality is inconsistent with the reactant and product chirality, as shown for an example reaction in Fig.~\ref{fig:rxn82}. The generated TS with FK, in contrast, depicts the correct chirality and, importantly, the correct reaction mechanism. The PCE and TCE error modalities in the ground truth data thus arises from incorrect database entries in some cases.
\begin{figure}[h]
    \centering
    \includegraphics[width=1.0\linewidth]{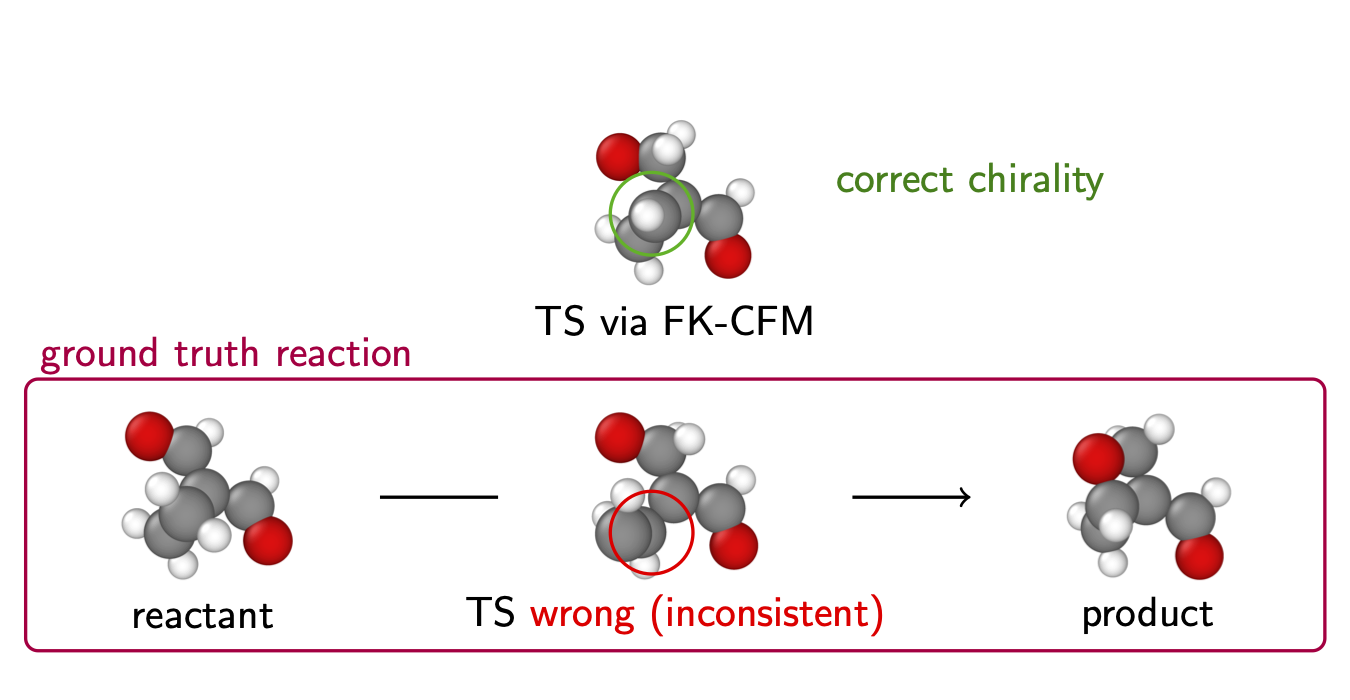}
    \caption{Exemplary ground-truth reaction with a wrong TS. Since the highlighted carbon atom sits on top (out of the plane) in the reactant and product, it cannot be below (under the plane) for the TS. The predicted TS by FK-CFM corrects this inconsistency.}
    \label{fig:rxn82}
\end{figure}

\end{document}